\newif\ifarxiv
\arxivtrue
\ifarxiv
    \documentclass[twocolumn]{article}
\else
    \documentclass{ieeeaccess}
    \usepackage[colorlinks=true, allcolors=blue]{hyperref}
\fi
\usepackage{epsfig}
\usepackage{booktabs} 
\usepackage{threeparttable}
\usepackage{amssymb,amsmath,amsthm,multirow}
\usepackage{amsfonts,graphicx}
\usepackage{float}
\usepackage{caption}
\usepackage{subcaption}
\usepackage{algorithm}
\usepackage{algpseudocode}
\usepackage{url}
\usepackage{stfloats}
\usepackage{tabularx}
\usepackage{placeins}

\newtheorem{proposition}{Proposition}

\newcommand{\paperabstract}{Counterfactual explanations provide local, interpretable insight by identifying changes to an input that would alter its assigned outcome. Although well established in supervised learning, their extension to clustering is less direct, since cluster assignments are unlabeled and governed by the geometry of the partition. This paper introduces VoICE, a Voronoi-Induced Counterfactual Explainability framework for feature-weighted $k$-means clustering. Rather than treating cluster change as a crossing of a single pairwise centroid boundary, VoICE formulates counterfactual generation as projection onto the full weighted Voronoi region of a target cluster, incorporating feature weights directly into both the clustering geometry and the counterfactual objective to yield least-cost and parsimonious explanations under actionability constraints. Target regions are further intersected with data-derived bounds and homothetically contracted towards their centroids, limiting extrapolation and boundary sensitivity. VoICE consistently produces valid target-cluster membership, across several benchmark datasets, where the leading pairwise baseline does not. \textbf{Software:} https://github.com/rickfawley/VoICE.}

\begin{document}
\title{Counterfactuals for Feature-Weighted Clustering}
\ifarxiv
    \date{}
    \author{Richard J. Fawley\thanks{\texttt{rf23433@essex.ac.uk}}\\
    Renato Cordeiro de Amorim\thanks{Corresponding author, \texttt{r.amorim@essex.ac.uk}} \\
    \textit{\small School of Computer Science and Electronic Engineering, University of Essex, Wivenhoe, UK.}
    }

    \twocolumn[
    \maketitle
    \begin{abstract}
    \paperabstract
    \end{abstract}
    \vspace{1em}
    \vspace{1em}
    ]
\else
    \history{Date of publication xxxx 00, 0000, date of current version xxxx 00, 0000.}
    \doi{10.1109/ACCESS.2026.DOI}
    \author{\uppercase{Richard J. Fawley}\authorrefmark{*},
    \uppercase{Renato  Cordeiro  de  Amorim}}
    \address{Computer Science and Electrical Engineering Department, University of Essex, Wivenhoe, UK}
    \address{Contact e-mails: rf23433@essex.ac.uk, r.amorim@essex.ac.uk.}    
    \begin{abstract}
    \paperabstract
    \end{abstract}
    \maketitle
\fi

\section{Introduction}

Explainable Artificial Intelligence (XAI) aims to improve the transparency and interpretability of machine learning models by providing insight into their outputs. Among the various explanation paradigms, \emph{counterfactual explanations} have emerged as a particularly intuitive and actionable approach.

Given a dataset \(X=\{x_1, \ldots, x_n\}\) with each \(x_i \in X\) described over \(d\) features, a counterfactual explanation identifies the minimal modification to an observation \(x_i \in X\) that changes its assigned outcome \cite{wachter2017counterfactual,guidotti2024counterfactual}. This is typically formalised as finding a perturbed instance whose prediction differs from the original while remaining as close as possible under a chosen distance metric. For example, in a loan approval setting, a counterfactual explanation may indicate the minimal changes required for an applicant to change from a rejected to an approved outcome, whereas a parsimonious counterfactual explanation may indicate a minimal increase in income alone. 

Explainability methods in XAI are often categorised as either global or local. Global approaches aim to summarise the overall behaviour of a model, whereas local methods focus on explaining individual predictions. Counterfactual explanations belong to the latter category, providing insight into specific outcomes by identifying the minimal changes required to alter a model's decision. Practical counterfactual explanations are often required to satisfy additional properties, such as \emph{actionability}, where only certain features may be modified, and \emph{plausibility}, ensuring that generated explanations lie within realistic, data-supported regions of the feature space. 

Clustering algorithms aim to group observations in such a way that those within the same cluster (i.e., group) are more similar to one another, according to some selected similarity measure, than those between clusters. However, not all clustering methods produce strict partitions of all available data. For example, some density-based methods allow for noise that is not explicitly allocated to any cluster, fuzzy clustering assigns to each observation a degree of membership between 0 and 1 thereby allowing multiple cluster memberships, and probabilistic mixture models assign each observation a probability distribution over clusters rather than a single hard assignment \cite{ester1996dbscan, campello2013hdbscan, bezdek1984fcm, mclachlan2000finite}. Among clustering methods, $k$-means \cite{macqueen1967} is one of the most widely used \cite{ikotun2023k,sinaga2020unsupervised}, assigning each observation to the cluster whose centre (the centroid) is the nearest. This induces a geometric structure in which cluster membership changes occur when an observation crosses a boundary between clusters. Counterfactual explanations can therefore be interpreted as minimal changes to the feature values required to cross such boundaries.

Despite this geometric intuition, generating counterfactual explanations for clustering presents challenges. First, explanations may be highly sensitive near cluster boundaries, where small perturbations can lead to qualitatively different outcomes. Second, controlling the number of features involved in a counterfactual is difficult, limiting interpretability and actionability. Third, counterfactuals can lie in regions of the feature space unsupported by the observed data. These challenges are further amplified in feature-weighted clustering settings as counterfactual approaches provide no mechanism for incorporating feature importance (for details, see Section~\ref{sec:clustering_and_fw}).

In this paper, we introduce a method for generating counterfactual explanations for feature-weighted $k$-means, where each feature is assigned an independent non-negative weight. We formulate counterfactual generation as a constrained optimisation problem over the Voronoi cell (the region of space containing all points closer to a given centroid than to any other) of a target cluster, enabling explicit control over both perturbation magnitude and sparsity. Feature weights are incorporated directly into both the geometry and the optimisation objective, providing a principled mechanism for optimising the construction of parsimonious explanations. To the best of our knowledge, counterfactual explanations have not been studied in the context of feature-weighted $k$-means clustering, despite the widespread use of such algorithms.

To address instability, a well documented issue in clustering under small perturbations \cite{vonluxburg2006stability, bendavid2006stability}, we introduce a robustness-aware mechanism based on homothetic contraction, which scales a bounded target Voronoi region towards its centroid. This contraction yields a compact, empirically bounded subset of the decision region that reduces exposure to boundary-adjacent areas and removes unbounded directions. Restricting counterfactual generation to this contracted region provides explicit geometric control over boundary sensitivity and supports more locally stable explanations
\cite{shalev2014understanding, bishop2006pattern}. To summarise, this paper makes the following contributions:

\begin{itemize}
    \item To the best of our knowledge, the first mechanism to incorporate feature weights directly into counterfactual explanations for clustering, shaping both the clustering geometry and the counterfactual objective;
    \item An extension of clustering counterfactuals from pairwise source-target boundary projection to full weighted Voronoi-region projection, with counterfactual validity ensured by construction;
    \item A constrained optimisation framework that restricts counterfactual generation to the observed data under actionability constraints, with a homothetic contraction mechanism providing explicit geometric control over robustness;
    \item A parsimonious search procedure based on ranked actionable feature subsets and minimal intervention cardinality.
\end{itemize}

\section{Background}

This section reviews the key concepts and prior work that underpin our proposed method. We begin by discussing feature weighting in clustering and its role in shaping the geometry of cluster assignments. We then examine existing approaches to counterfactual explanations in clustering, highlighting their limitations. 

\subsection{Feature weighting in clustering}
\label{sec:clustering_and_fw}

Clustering algorithms partition a dataset \(X\) into \(k\) clusters \(C_1, \ldots, C_k\) based on a similarity or distance measure. In centroid-based methods such as $k$-means, each cluster \(C_j\) is represented by a centroid \(m_j \in \mathbb{R}^d\). Each observation \(x_i \in X\) is assigned to the cluster represented by the centroid that is nearest to \(x_i\). In the case of $k$-means, specifically, this minimises
\[
W = \sum_{j = 1}^k \sum_{x_i \in C_j} \sum_{v=1}^d (x_{iv} - m_{jv})^2,
\]
where,
\[
m_{jv} = \frac{1}{|C_j|} \sum_{x_i \in C_j} x_{iv}.
\]

The procedure itself follows three steps: (i) select \(k\) observations from \(X\) uniformly at random, and copy their values to \(m_1, \ldots, m_k\); (ii) for each observation \(x_i \in X\) identify its nearest centroid \(m_j\) and assign \(x_i\) to \(C_j\); (iii) update each centroid to the component-wise mean of the observations in its cluster. These steps are repeated until convergence.

Although popular, $k$-means does have some known drawbacks. For instance, its final clustering strongly depends on the initial centroids (chosen at random), the number of clusters \(k\) must be known beforehand, and it assumes that all features are equally relevant. Here, we are particularly interested in the latter. In many real-world datasets, only a subset of features contributes meaningfully to the underlying cluster structure, and assigning equal importance to all features may obscure informative patterns \cite{deng2016survey,chakraborty2022}. Also, even among relevant features there may be different degrees of relevance, which should be taken into account. Feature weighting extends clustering algorithms to account for such heterogeneous feature relevance. 

Feature weighting in clustering is a popular research area (see \cite{hancer2020survey,deng2016survey,huang2005ewkm,chakraborty2022}). Existing approaches aim to modify the contribution of each feature to the clustering objective, typically by incorporating weights into the distance function used to define cluster assignments. Formally, they assign a non-negative weight \(\omega_v\) to each feature \(v=1, \ldots, d\). Each \(\omega_v\) represents the degree of relevance of feature \(v\). In this work, we are concerned with the use of such weights within counterfactual generation, rather than with the specific mechanisms by which they are estimated.

\subsection{Counterfactual explanations}
\label{sec:background_counterfactual}

%Changing order, going from general to the specific area of clustering
Optimisation-based methods have been widely adopted to generate valid and diverse counterfactual explanations \cite{guidotti2024counterfactual,prado2024survey}. A prominent example is DiCE \cite{mothilal2020dice}, which formulates counterfactual generation as a multi-objective optimisation problem balancing proximity to the factual instance with diversity across multiple counterfactuals. Alternative approaches generate counterfactuals by moving instances towards representative prototypes or along rule-based paths, trading optimality for interpretability \cite{vanlooveren2021interpretable,ustun2019actionable}. While such methods are flexible and sometimes model-agnostic, they are primarily developed for supervised settings and do not explicitly account for the geometric structure of clustering decision regions.

There has been limited research exploring counterfactual explanations in clustering settings, often in distinct and largely disconnected directions. For instance, counterfactual analysis has been used to quantify the cost of enforcing fairness constraints in $k$-means clustering, comparing similarity-based measures with counterfactual-based costs of reassigning individuals between clusters \cite{karra2025fairness}. This line of work highlights how counterfactuals can reveal asymmetric impacts across groups and identify features acting as proxies for sensitive attributes.

Counterfactuals for Clustering (CFCLUST) \cite{vardakas2025counterfactual} is, to our knowledge, the first formal framework for generating counterfactual explanations for clustering algorithms. It considers both $k$-means clustering and Gaussian mixture models, we focus on the former as our proposed method is grounded in $k$-means (for details see Section \ref{sec:method}). CFCLUST exploits the fact that cluster assignments are determined by proximity to centroids, so that the boundary between two clusters is given by the hyperplane equidistant from the corresponding centroids. Under this geometric formulation, the counterfactual can be obtained by projecting the factual instance onto the separating hyperplane between clusters. This framework proposes valuable practical constraints that are important in real-world settings:
\begin{itemize}
\item \emph{Actionability constraints}, modelled as binary feature masks that allow certain features to remain immutable, reflecting attributes that cannot be changed;
\item \emph{Plausibility constraints}, modelled as a boundary penetration scalar $\epsilon$, move the counterfactual away from the source–target boundary and further into the target side of the pairwise model.
\end{itemize}

Pairwise-boundary counterfactuals for $k$-means rely on the separating hyperplane between a source and target centroid. This is sufficient to describe a binary transition, but in multi-cluster settings it does not in general guarantee membership of the intended target Voronoi cell. In this paper, we extend this geometric perspective by considering the full Voronoi polytope induced by multiple neighbouring centroids, rather than a single pairwise hyperplane, and build on it with additional mechanisms for robustness and interpretability.

\section{Proposed Method}
\label{sec:method}
%\label{sec:sec3_intro} No reason for two labels

This section presents our proposed framework, VoICE, for generating counterfactual explanations in a clustering setting using a feature-weighted version of $k$-means. We begin by introducing weighted Voronoi regions as target counterfactual sets, followed by the construction of bounded feasible target regions and their robustness refinement via homothetic contraction. We then formulate counterfactual generation as a constrained optimisation problem over these regions, before introducing parsimonious counterfactual explanations. Finally, we summarise our proposed method.

\subsection{Voronoi regions as target clusters}
\label{subsec:weighted_voronoi}

%Don't introduce unnecessary notation. Too much notation = paper difficult to read
We model counterfactual generation in a feature-weighted $k$-means scenario as an optimisation problem over the decision region of a target cluster. Here, we are not particularly interested in what exact algorithm is being used as long as it generates a feature weight vector \(\omega = (\omega_1, \ldots, \omega_d)\), where each of its components, \(\omega_v\), denotes the weight associated with a feature \(v\) and \(\omega_v \geq 0\). With this we can define the weighted squared Euclidean distance
\[
D_{\omega}(x_i,m_j)=\sum_{v=1}^d \omega_v (x_{iv}-m_{jv})^2,
\]
for any \(x_i \in X\) and centroid \(m_j\). We can then represent the decision region of a cluster \(C_t\) with centroid \(m_t\) as the Voronoi cell
\[
V_t^{(\omega)}=
\left\{
x\in \mathbb{R}^d : D_{\omega}(x,m_t)\le D_{\omega}(x,m_j)\;\forall j\neq t
\right\}.
\]

Expanding the inequalities and cancelling quadratic terms yields the half-space representation
\begin{equation}
\label{eq:weighted_voronoi_halfspace}
\begin{aligned}
V_t^{(\omega)}
&=
\bigcap_{j\neq t}
\Biggl\{
x \;:\; (\Omega(m_j-m_t))^\top x \\
&\qquad\le \frac{1}{2}\Bigl(
m_j^\top\Omega m_j - m_t^\top\Omega m_t
\Bigr)
\Biggr\},
\end{aligned}
\end{equation}
where, $\Omega = \mathrm{diag}(\omega_1,\ldots,\omega_d)$. Thus, the target cluster is represented by a convex polyhedral region in feature space, and counterfactual generation can be formulated as an optimisation problem over this region.

The half-space representation in \eqref{eq:weighted_voronoi_halfspace} involves $k-1$ inequalities. However, only those inequalities corresponding to facet-sharing neighbours are required to define \(V_t^{(\omega)}\). Rather than incurring the computational cost of constructing a complete diagram, the Voronoi region of a single centroid \(m_t\) can be obtained by intersecting the half-spaces defined by pairwise bisectors with respect to a subset of relevant neighboring clusters~\cite{aurenhammer1991voronoi}. 

These neighbours can be efficiently identified using properties of the corresponding Delaunay triangulation (the dual graph of the Voronoi diagram, in which centroids are connected if their Voronoi cells share a boundary)~\cite{delaunay1934sphere,lee1980twodimensional}, thereby avoiding redundant comparisons with distant centroids that do not influence the cell boundary. As a result, localised Voronoi cell construction offers a computationally efficient alternative, with complexity that depends primarily on the number of effective neighbours rather than \(k\).

In the unweighted two-cluster case, with all features actionable, no contraction, and no additional bounded-domain restriction, projection onto the target Voronoi region reduces to projection onto the source–target bisector. Thus, in this special case, the proposed formulation recovers the $k$-means CFCLUST construction. For $k>2$, VoICE extends the pairwise construction to the full target Voronoi region.

\subsection{Bounded feasible target regions}
\label{subsec:bounded_feasible_regions}

Voronoi cells in $k$-means clustering may be unbounded, as centroids lying on the convex hull of the configuration induce regions that extend infinitely in directions where no competing centroid is closer. As a result, these regions can include parts of the feature space that are not supported by the observed data. To reduce extrapolation beyond the observed feature ranges, we first restrict the target Voronoi cell to a bounded, empirically defined feasible domain. Let $\mathcal{F} \subseteq \mathbb{R}^d$ denote a feasible region derived from observed data (i.e.\ feature-wise bounds). We define the bounded target region as
\[
\tilde V_t^{(\omega)} = V_t^{(\omega)} \cap \mathcal{F}.
\]
Notice that \(\mathcal{F}\) is a polyhedron. Since \(V_t^{(\omega)}\) is also a polyhedron by \eqref{eq:weighted_voronoi_halfspace}, their intersection \(\tilde V_t^{(\omega)}\) is a polyhedron and admits a half-space representation of the form
\[
\tilde V_t^{(\omega)} = \{x \in \mathbb{R}^d : Ax \le b\}.
\]

To improve robustness and reduce sensitivity near cluster boundaries, we apply a homothetic contraction, i.e., a uniform scaling of the region towards its centre (the cluster centroid, \(m_t\)), to this bounded region. 
For $\alpha \in (0,1]$, define
\[
\tilde V_t^{(\omega,\alpha)}
=
\left\{
m_t + \alpha(x - m_t) : x \in \tilde V_t^{(\omega)}
\right\}.
\]

Expressing \(Ax\leq b\) in terms of \(x = m_t+\alpha(y-m_t) \in \tilde V_t^{(\omega,\alpha)}\) for some \(y \in \tilde V_t^{(\omega)}\) leads to

\begin{equation}
\tilde V_t^{(\omega,\alpha)}
=
\{x \in \mathbb{R}^d : Ax \le (1-\alpha)Am_t + \alpha b\},
\label{eq:contracted_halfspace}
\end{equation}
showing the contraction scales each constraint towards the centroid $m_t$. Figure~\ref{fig:homothetic_contraction} illustrates this contraction. Smaller values of $\alpha$ yield stronger contraction, increasing the geometric margin from decision boundaries at the cost of excluding boundary-adjacent observations, while larger values retain a greater proportion of the original Voronoi region, with $\alpha=1$ recovering the original bounded cell.

\begin{figure}[t]
    \centering
    \includegraphics[width=0.75\linewidth]{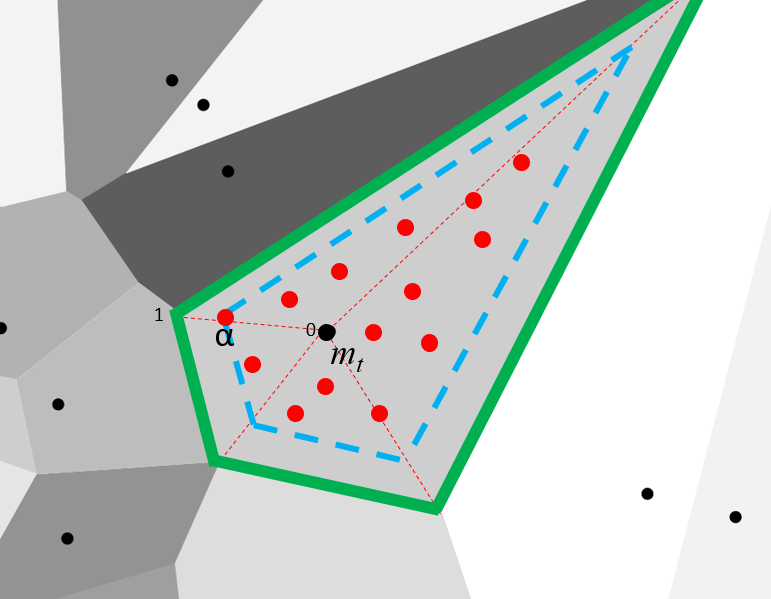}
    \caption{
    Homothetic contraction of a Voronoi cell. Green denotes the Voronoi region $\widetilde V_t^{(\omega)}$, while the dashed blue denotes the contracted region $\widetilde V_t^{(\omega,\alpha)}$ . The red points are observations in cluster \(C_t\). 
    }
    \label{fig:homothetic_contraction}
\end{figure}

For each \(x_i\in C_t\), define the point-wise contraction score
\[
\alpha_{i,t}
=
\inf\{\alpha\in[0,1]:x_i\in\widetilde V_t^{(\omega,\alpha)}\}.
\]

This score records how much of the bounded target region must be retained in order for observation \(x_i\) to remain inside the contracted region. The cluster-level full-retention factor is then
\[
\alpha_t^\star=\max_{x_i\in C_t}\alpha_{i,t}.
\]
Thus, \(\alpha_t^\star\) (the smallest contraction factor retaining all observed members of \(C_t\)) is determined by the most boundary-adjacent retained member of cluster \(C_t\), while the distribution of the point-wise scores \(\alpha_{i,t}\) provides a diagnostic summary of how centrally or peripherally
the assigned observations lie within the bounded target region. These point-wise scores are used later as contraction diagnostics.

The empirical interpretation of $\alpha$ depends on the metric used to construct the region. When feature weights are highly concentrated, $\alpha$ may describe compactness in a low-dimensional weighted subspace rather than compactness in the original feature space. We therefore report feature-weight concentration alongside $\alpha$-based contraction statistics (see Section~\ref{sec:results}).

\subsection{Counterfactual optimisation}
\label{subsec:cf_optimisation}

Let \(a\) be a factual observation assigned to a source cluster \(C_s\), and \(C_t\) be a target cluster with \(t \neq s\). Let \(A\subseteq\{1,\ldots,d\}\) denote the set of actionable features. When actionability constraints are imposed, features outside $A$ are held fixed at their factual values. We denote the corresponding affine subspace by
\[
E_A(a) = \{z \in \mathbb{R}^d : z_j = a_j \ \forall j \notin A\}.
\]

The optimisation is therefore carried out over the feasible target region, \(\widetilde{V}^{(\omega,\alpha_t^\star)}_t\), together with the equality constraints in \(E_A(a)\). 
%
%We denote the robust feasible target region associated with \(C_t\) by
%\[
%\mathcal{R}_t
%=
%\widetilde{V}^{(\omega,\alpha_t^\star)}_t .
%\]
%

To ensure uniqueness even when some feature weights are zero or negligible, we use the regularised weighted objective
\[
D_{\omega,\delta}(a,z)
=
D_\omega(a,z)+\delta\|z-a\|_2^2,
\qquad \delta>0,
\]
where \(\delta\) is chosen small enough to act only as a numerical tie-breaker. The counterfactual for \(a \in C_s\) with respect to target cluster \(C_t\) is then
\[ 
z_t^\star 
= 
\arg\min_{z\in \widetilde{V}^{(\omega,\alpha_t^\star)}_t\cap E_A(a)} 
D_{\omega,\delta}(a,z). 
\tag{3} 
\]

This incorporates feature weights directly into the optimisation objective while retaining a strictly convex projection problem.

If the target cluster \(C_t\) is not specified in advance, we solve (3) for each candidate target \(t\neq s\) and select the counterfactual with minimum cost among feasible solutions. 

\subsection{Parsimonious counterfactuals}
\label{subsec:parsimonious_counterfactuals}

While the optimisation problem defined in Section~\ref{subsec:cf_optimisation} identifies a feasible counterfactual explanation, practical interpretability often requires explanations that modify as few actionable features as possible. To address this requirement, our proposed framework incorporates ranked feature subsets and searches for the most parsimonious feasible counterfactual.

Let $S_r$ denote the set of the $r$ highest-ranked actionable features, ranked in descending order of $\omega_v$, with $r \leq |A|$. The set $S_r$ represents the subset of actionable features permitted to vary during optimisation, with all remaining features fixed at their factual values.

For a given target cluster \(C_t\), let \(z_t^{(r)}\) denote the solution of the counterfactual optimisation problem when variation is restricted to features in \(S_r\). If no feasible solution exists for a given \(r\), the corresponding optimisation problem is considered infeasible. The proposed framework searches the ranked prefixes sequentially until either a feasible counterfactual is obtained or all actionable prefixes have been exhausted. That is, it starts with \(S_1\), incrementing \(r\) until a feasible solution is found or \(r=|A|\) has been tested. When at least one feasible prefix exists, this yields the smallest actionable feature subset capable of producing a valid counterfactual. This allows us to define the \emph{minimal intervention cardinality}
\[
r_t^{\star} = \min\{r \in \{1,\ldots,|A|\} \mid z_t^{(r)} \text{ is feasible}\}.
\]
If no feasible prefix exists, then no parsimonious counterfactual is returned for the corresponding factual--target pair.

The quantity \(r_t^{\star}\), when defined, represents the smallest number of ranked actionable features required to generate a feasible counterfactual explanation with target cluster \(C_t\). Consequently, lower values of $r_t^\star$ correspond to more parsimonious explanations, while larger values indicate that the desired cluster transition requires variation across a broader set of actionable features.

The associated counterfactual,
\[
z_t^\star = z_t^{(r_t^\star)},
\]
is referred to as the \emph{most parsimonious counterfactual explanation}. 

The minimal intervention cardinality provides a natural measure of explanation complexity and forms the basis of the model-level parsimony metrics introduced in Section~\ref{sec:evaluation_metrics}. When the target cluster is clear from context, or when results are aggregated across sampled factual--target pairs, we write this quantity as \(r^\star\). Aggregated statistics such as the mean and median values of \(r^\star\) are computed over feasible factual--target pairs and quantify the average number of ranked actionable features required to obtain feasible counterfactual explanations across a dataset.

\subsection{Directional counterfactual ranges}
\label{subsec:directional_ranges}

The optimisation problems above identify a minimal feasible counterfactual point. However, once such a point has been found, it is also useful to characterise how far one may continue to move in the same counterfactual direction while remaining feasible. Let \(z_t^\star\) be a feasible counterfactual for factual observation \(a\), and define the counterfactual direction 
\[ 
d_t=z_t^\star-a. 
\] 
For \(\lambda\geq0\), consider the ray 
\[
z_t(\lambda)=a+\lambda d_t. 
\] 
By construction, \(z_t(1)=z_t^\star\). The admissible directional range is 
\[ 
\mathcal I_t(a,d_t) = \{\lambda\geq1:z_t(\lambda)\in \widetilde{V}^{(\omega,\alpha_t^\star)}_t\cap E_A(a)\}. 
\] 
Since \(\widetilde{V}^{(\omega,\alpha_t^\star)}_t\cap E_A(a)\) is convex, this set is an interval of the form 
\[ 
[1,\rho_t], 
\] whenever \(z_t^\star\in \widetilde{V}^{(\omega,\alpha_t^\star)}_t\cap E_A(a)\). The value \(\rho_t\) is the largest admissible intervention magnitude along the same direction before the path exits the contracted feasible target region or violates the actionability constraints. The associated directional tolerance is 
\[
\tau_t=\rho_t-1. 
\] 
Larger values of \(\tau_t\) indicate that the counterfactual remains feasible over a wider range of intervention magnitudes along the same direction, while values close to zero indicate that the minimal counterfactual lies close to a limiting boundary.

\subsection{Algorithmic summary}
\label{subsec:algorithmic_summary}

We now summarise the proposed framework through the following algorithms.

\paragraph{Counterfactuals as Voronoi projections.}
This baseline formulation computes counterfactuals by projecting the factual observation onto the Voronoi cell of a target cluster under Euclidean distance.

\begin{algorithm}[H]
\caption{VoICE (unweighted)}
\label{alg:baseline_cf}
\begin{algorithmic}[1]
\Require Factual observation \(a\), source cluster \(C_s\),
centroids \(\{m_1,\ldots,m_k\}\), optional target cluster index \(t\)
\Ensure Counterfactual $z^\star$

\If{$t$ is specified}
    \State $\mathcal{T}\leftarrow \{t\}$
\Else
    \State $\mathcal{T}\leftarrow \{1,\ldots,k\}\setminus\{s\}$
\EndIf

\State Initialise candidate set $\mathcal{Z}\leftarrow\emptyset$

\For{each $t\in\mathcal{T}$}
    \State Construct Voronoi cell $V_t$
    \State Solve projection $z_t^\star = \arg\min_{z \in V_t} \|z-a\|^2$
    \State Add $z_t^\star$ to $\mathcal{Z}$
\EndFor
\State \textbf{return} $z^\star = \arg\min_{z\in\mathcal{Z}} \|z-a\|^2$
\end{algorithmic}
\end{algorithm}

\paragraph{Robust weighted counterfactuals.}
This proposed method extends the previous baseline by incorporating feature weights, restricting the search to an empirically bounded feasible domain, and improving robustness through homothetic contraction of the target region.

\begin{algorithm}[H]
\caption{VoICE: Robust weighted counterfactuals}
\label{alg:voice_cf}
\begin{algorithmic}[1]
\Require Factual observation \(a\), source cluster \(C_s\), centroids \(\{m_1,\ldots,m_k\}\), feature weights \(\omega\), regularisation parameter \(\delta>0\), feasible domain \(\mathcal{F}\), actionable feature set \(A\), optional target cluster index \(t\)
\Ensure Counterfactual $z^\star$ or failure

\If{$t$ is specified}
    \State $\mathcal{T}\leftarrow \{t\}$
\Else
    \State $\mathcal{T}\leftarrow \{1,\ldots,k\}\setminus\{s\}$
\EndIf

\State Initialise candidate set $\mathcal{Z}\leftarrow\emptyset$

\For{each $t\in\mathcal{T}$}
    \State Construct the weighted Voronoi cell $V_t^{(\omega)}$
    \State Set $\tilde V_t^{(\omega)} = V_t^{(\omega)} \cap \mathcal{F}$
    \State Compute contracted region $\tilde V_t^{(\omega,\alpha_t^\star)}$
    \State Solve projection 
    \[ 
    z_t^\star 
    = 
    \arg\min_{z\in \tilde V_t^{(\omega,\alpha_t^\star)}\cap E_A(a)} 
    \{D_{\omega,\delta}(a,z)\}. 
    \]
    \If{feasible}
        \State Add $z_t^\star$ to $\mathcal{Z}$
    \EndIf
\EndFor

\If{$\mathcal{Z}=\emptyset$}
    \State \textbf{return} failure
\Else
    \State \textbf{return}
    \(
    z^\star =
    \arg\min_{z_t^\star\in\mathcal{Z}}
    \{D_{\omega,\delta}(a,z_t^\star)\}
    \)
\EndIf
\end{algorithmic}
\end{algorithm}

We now show that, for a fixed target cluster \(C_t\), the regularised projection problem solved within Algorithm~\ref{alg:voice_cf} admits a unique solution.

\begin{proposition}
Let $\widetilde V_t^{(\omega,\alpha_t^\star)}\cap E_A(a) \neq \emptyset$. Then, 
\[
z_t^\star
=
\arg\min_{z\in \widetilde V_t^{(\omega,\alpha_t^\star)}\cap E_A(a)}
D_{\omega,\delta}(a,z)
\]
admits a unique solution.
\end{proposition}

\begin{proof}
The region \(\widetilde{V}^{(\omega,\alpha_t^\star)}_t\) is the intersection of finitely many closed half-spaces and feature-wise bounds, followed by homothetic contraction. Hence, \(\widetilde V_t^{(\omega,\alpha_t^\star)}\) is compact and convex. The set \(E_A(a)\) is an affine subspace, so \(\widetilde V_t^{(\omega,\alpha_t^\star)}\cap E_A(a)\) is also compact and convex. 

The objective \(D_{\omega,\delta}(a,z)\) is continuous on \(\widetilde V_t^{(\omega,\alpha_t^\star)}\cap E_A(a)\). Since \(\widetilde V_t^{(\omega,\alpha_t^\star)}\cap E_A(a)\) is compact and non-empty, a minimiser exists. Furthermore,
\[
D_{\omega,\delta}(a,z)
=
(z-a)^\top(\Omega+\delta I)(z-a).
\]
Because \(\omega_v\geq 0\) for all \(v\) and \(\delta>0\), the matrix
\(\Omega+\delta I\) is positive definite. Hence
\(D_{\omega,\delta}(a,z)\) is strictly convex in \(z\). A strictly convex
function has at most one minimiser over a convex set. Therefore the minimiser is
unique.
\end{proof}

\paragraph{A parsimonious extension}
The proposed framework admits a natural parsimonious extension, which searches for the minimal subset of actionable features required to produce a feasible counterfactual. Starting from the single most important actionable feature, the method incrementally adds features in descending order of importance until a feasible solution is found, terminating in at most $|A|$ steps (see Section~\ref{subsec:parsimonious_counterfactuals} for details). Feature importance is determined by the non-negative weight, \(\omega_v\), assigned to each actionable feature. 

\section{Experimental Setup}
\label{sec:experimental_evaluation}

This section describes the experimental methodology we use to evaluate our proposed framework.

\subsection{Comparison with CFCLUST}

We compare VoICE (unweighted), described in Algorithm \ref{alg:baseline_cf}, against CFCLUST (for details, see Section \ref{sec:background_counterfactual}), to our knowledge the only existing framework for counterfactual explanations in $k$-means clustering. Following the original experiments introducing CFCLUST~\cite{vardakas2025counterfactual}, we generate counterfactuals for 50 randomly selected factual observations from the source cluster, with the target cluster also selected at random. The underlying clustering model is unweighted $k$-means, with ranked feature restrictions and $\alpha$-contraction excluded, so that the comparison isolates the geometric difference between the two feasible regions.

For each factual observation $a$ assigned to source cluster $C_s$ and each selected target cluster $C_t$, we generated two counterfactuals under this shared setup, using the same squared Euclidean cost and no plausibility shift. The CFCLUST counterfactual was obtained by projecting $a$ onto the pairwise bisecting hyperplane between the centroids $m_s$ and $m_t$, while the VoICE (unweighted) counterfactual was obtained by projecting $a$ onto the full target Voronoi cell $V_t$.

In our experiments, we measure validity. That is, the percentage of counterfactuals lying inside the target \(V_t\). Given VoICE (unweighted) optimises directly over $V_t$, its validity is guaranteed to be 100\% by construction regardless of the number of clusters \(k\). Our comparison therefore quantifies how often the CFCLUST counterfactual fails this same criterion. Where it does, we report repair cost, the squared distance from the invalid counterfactual to the corresponding valid VoICE counterfactual, measuring the extent to which CFCLUST can report an apparently lower-cost counterfactual by stopping outside the true target Voronoi region when $k>2$.

\subsection{Clustering and weighting modes}

The proposed framework operates on clustering structures induced by $k$-means, with feature weights optionally incorporated to reflect the relative importance of each dimension in the cluster geometry. In our experiments, weights are obtained using Shapley Reweighted $k$-means (SHARK) \cite{fawley2026shapley}, which decomposes the $k$-means objective into per-feature within-cluster dispersion contributions and assigns weights inversely related to them, down-weighting noisy or high-variance features. These weights are incorporated into the weighted distance function defined in Section~\ref{sec:method}. We adopt SHARK because it introduces no additional hyperparameters beyond those required by $k$-means, preserves the computational complexity of standard $k$-means, and aligns naturally with the geometric formulation used here. However, our framework is not restricted to SHARK. Any method producing non-negative feature weights may be used, including entropy-based \cite{huang2005ewkm}, regularisation-based \cite{chakraborty2022}, or adaptive weighting schemes \cite{liu2016fwsa}.

We evaluate three modes. In Unweighted $k$-means, equal feature weights are used for clustering, ranking, and counterfactual optimisation. In Ranked $k$-means, $k$-means centroids and equal-weight Voronoi geometry are retained, but SHARK-derived weights rank features for the parsimonious search. In Ranked + weighted SHARK, SHARK weights are used throughout, for clustering, ranking, Voronoi geometry, and the counterfactual objective.

Note that Unweighted $k$-means differs from VoICE (unweighted, Algorithm~\ref{alg:baseline_cf}). The former retains the full feasibility bounds, contraction, and optimisation machinery of Algorithm~\ref{alg:voice_cf}, simply with uniform feature weights, whereas the latter omits these mechanisms entirely.

Given feature weighting can alter centroid locations, Voronoi geometry, and cluster assignments, the same factual observation may be assigned to different source clusters, and the corresponding target regions may differ, across clustering formulations. Comparisons between clustering formulations should therefore be interpreted as comparisons between complete modelling pipelines, rather than as comparisons of explanation mechanisms applied to the same underlying clustering.

\subsection{Feasibility and actionability}

Counterfactual generation is constrained to a feasible region encoding plausibility and actionability. In our experiments, $\mathcal{F}$ is constructed using feature-wise bounds derived from the observed data, $\mathcal{F} = [\min(X), \max(X)]$, taken component-wise, which keeps counterfactuals within the empirical range of each feature and avoids extrapolation. This construction is simple and computationally efficient; alternatives such as convex hull approximations or neighbourhood-based constraints may be used depending on the application.

Actionability is incorporated through binary feature masks defined by the semantic interpretation of each feature, for example holding demographic attributes such as age immutable while allowing financial attributes to vary, formally enforced via the affine constraint $E_A(a)$ from Section~\ref{subsec:cf_optimisation}. For each dataset we consider both fully actionable and partially constrained masks. For non-interventional datasets, such masks should be interpreted as restrictions on profile-shift explanations rather than claims about real-world recourse. For mixed-type diagnostic datasets, semantic masks are reported descriptively, but categorical and one-hot encoded variables require additional validity constraints beyond the present weighted Euclidean formulation.

\subsection{Datasets}
\label{sec:datasets}

We evaluate the proposed framework using a two-tier dataset strategy, summarised in Table~\ref{tab:datasets} (including illustrative benchmark datasets present in the recent literature~\cite{vardakas2025counterfactual}). The first tier contains the primary evaluation datasets used for the main quantitative comparison. These datasets were selected because the weighted geometry remains sufficiently well behaved for the resulting counterfactuals, contraction statistics, and feature-ranking behaviour to be interpreted as properties of a meaningful multivariate clustering model. The second tier contains mixed-type diagnostic stress-test datasets. These datasets are retained to characterise failure modes of feature-weighted Euclidean geometry, particularly cases where feature weights collapse onto one or two binary or one-hot encoded indicators.

\begin{table}[!t] 
\centering 
\scriptsize 
\caption{Dataset strategy used in the evaluation. Primary datasets are used for the main quantitative comparison, while diagnostic stress-test datasets are retained to analyse feature-weight concentration and mixed-type degeneracy.} 
\label{tab:datasets} 
\begin{tabular}{lcccc} 
\toprule 
&&& & \textbf{Degeneracy} \\ 
\textbf{Dataset} & \textbf{$n$} & \textbf{$d$} & \textbf{$k$}   & \textbf{risk}\\
\midrule 
\multicolumn{5}{l}{\textit{Evaluation role: Primary}}\\
Iris~\cite{fisher_iris} & 150 & 4 & 3 & Low \\ 
Wine~\cite{aeberhard_wine} & 178 & 13 & 3 &Low \\ 
Palmer Penguins~\cite{horst2020palmerpenguins} & 342 & 4 & 3 & Low \\
Breast Cancer~\cite{street1993nuclear} & 569 & 30 & 2  & Low \\ 
Wholesale Cust.~\cite{uci_wholesale_customers}$^{a}$ & 440 & 6 & 3 & Moderate \\ 
\midrule 
\multicolumn{5}{l}{\textit{Evaluation role: Diagnostic stress test}}\\
Diabetes~\cite{cdc_diabetes_health_indicators} & 253,680 & 20 & 3 & High\\ 
Obesity~\cite{obesity_dataset_uci} & 2,111 & 22 & 7  & High \\ 
German Credit~\cite{german_credit_data} & 1,000 & 48 & 2  & High \\ 
Heart Failure~\cite{uci_heart_failure_2020} & 299 & 12 & 2  & High \\ 
Student Perf.~\cite{cortez2008student} & 649 & 39 & 5  & High \\ 
\bottomrule 
\end{tabular} 
\begin{tablenotes}[flushleft]
\footnotesize
\item $^{a}$Channel and region were excluded from the data, while region was retained as the ground-truth label.
\end{tablenotes}
\end{table}

\begin{table}[!t]
\centering
\scriptsize
\setlength{\tabcolsep}{5pt}
\renewcommand{\arraystretch}{1.05}
\caption{Clustering-solution diagnostics for the experimental datasets. ARI is reported only where external reference labels are available and is used as an external diagnostic rather than an optimisation criterion.}
\label{tab:clustering_diagnostics}
\begin{tabular}{lcccc}
\toprule
&\multicolumn{2}{c}{\textbf{$k$-means}} &\multicolumn{2}{c}{\textbf{SHARK}}\\
\cmidrule(lr){2-3}
\cmidrule(lr){4-5}
\textbf{Dataset} &
\textbf{ARI} &
\textbf{Inertia} &
\textbf{ARI} &
\textbf{Inertia} \\
\midrule

Iris
& 0.620
& 140.97
& 0.886
& 15.87 \\

Wine
& 0.897
& 1,277.93
& 0.822
& 81.15 \\

Palmer Penguins
& 0.793
& 379.39
& 0.609
& 86.41 \\

Breast Cancer
& 0.654
& 11,595.53
& 0.718
& 341.58 \\

Wholesale Customers
& 0.028
& 1,608.43
& 0.000
& 211.09\\

\midrule

Diabetes
& 0.149
& 4,221,626.00
& 0.167
& $1.90 \times 10^{-11}$$^{a}$ \\

Obesity
& 0.101
& 29,808.39
& 0.156
& $1.00 \times 10^{-11}$ $^{a}$\\

German Credit
& 0.000
& 45,797.93
& 0.047
& $4.70 \times 10^{-11}$ $^{a}$\\

Heart Failure
& -0.004
& 3,201.12
& -0.005
& $1.10 \times 10^{-11}$ $^{a}$ \\

Student Performance
& 0.027
& 21,476.45
& 0.036
& $1.85 \times 10^{-11}$ $^{a}$\\

\bottomrule
\end{tabular}
\begin{tablenotes}[flushleft]
\footnotesize
\item $^{a}$ The near-zero SHARK inertia indicates degeneracy caused by extreme feature-weight concentration, not near-perfect multivariate clustering.
\end{tablenotes}
\end{table}

The primary evaluation datasets (see Table~\ref{tab:datasets}) are well-understood and have low- to moderate-dimensional structure. This enables qualitative analysis of Voronoi geometry, homothetic contraction, and directional counterfactual ranges. Palmer Penguins provides a similarly interpretable morphological dataset with partial cluster overlap. Breast Cancer provides a higher-dimensional numerical benchmark, making it useful for evaluating feature-normalised parsimony. Wholesale Customers provides a behavioural segmentation setting based on six annual-spending variables. Channel and Region were excluded from the clustering matrix, while Region was retained separately as the ground-truth label for external evaluation. The clustering used $k=3$, corresponding to the three Region classes. Counterfactuals therefore represent changes in purchasing profiles rather than changes in geographical region or customer channel.

The diagnostic stress-test datasets contain binary, ordinal, or one-hot encoded variables with direct semantic interpretation. Although such datasets are relevant to practical recourse settings, they also expose an important limitation of the present weighted Euclidean formulation. Under SHARK weighting, several of these datasets exhibit severe feature-weight concentration, causing the effective geometry to collapse onto a single binary feature or a small categorical block. In such cases, near-zero contraction values should not be interpreted as evidence of compact multivariate cluster structure, but rather as a diagnostic indication of feature-weight degeneracy.

Accordingly, primary aggregate results are reported on the first tier of datasets, with the diagnostic stress-test results reported separately in Section~\ref{sec:results}.

\subsection{Evaluation metrics}
\label{sec:evaluation_metrics} 

We evaluate the proposed framework using metrics designed to assess validity, feasibility, intervention cost, parsimony, directional robustness, and feature-weight behaviour. 

\paragraph{Validity.} A counterfactual is considered valid if it is assigned to the intended target cluster under the clustering function used by the corresponding mode. We report the proportion of generated counterfactuals that satisfy this condition. 

\paragraph{Feasibility.} We measure feasibility as the proportion of sampled factual--target pairs for which a feasible counterfactual can be generated under the specified actionability and feasibility constraints. 

\paragraph{Distance and cost.} We evaluate perturbation magnitude using the weighted squared distance $D_\omega(a,z^\star)$ between the factual observation $a$ and the counterfactual $z^\star$, reflecting intervention cost in the weighted feature space. 

\paragraph{Sparsity and parsimony.} We evaluate explanation parsimony using two complementary measures: the number of feature values differing between the factual observation and the counterfactual (with a feature counted as changed when \(|z^\star_j-a_j|>\varepsilon\) for small tolerance \(\varepsilon\)), and the minimal intervention cardinality \(r^\star\) (Section~\ref{subsec:parsimonious_counterfactuals}). To compare datasets of different dimensionality, we additionally report the feature-normalised cardinality \(r^\star/d\).

\paragraph{Directional tolerance.} We report the directional tolerance \(\tau=\rho-1\) (see Section~\ref{subsec:directional_ranges}), which quantifies the range of intervention magnitudes for which the counterfactual direction remains admissible.

\paragraph{Contraction diagnostics.} For each cluster, we report the full-retention contraction factor \(\alpha^\star_t\) (see Section~\ref{subsec:bounded_feasible_regions}), together with the minimum, median, and maximum point-wise contraction scores \(\alpha_{i,t}\). The median score helps distinguish clusters whose members are generally boundary-adjacent from clusters with only a small number of boundary-adjacent observations.

\paragraph{Feature-weight concentration.} Because feature weights determine the weighted Voronoi geometry, we report feature-weight concentration diagnostics, including the maximum feature weight, the effective number of active features, and the dominant feature or feature block. These diagnostics are used to identify cases where the weighted geometry collapses onto one or two binary or one-hot encoded indicators. 

For feature weights normalised so that \(\sum_{j=1}^d\omega_j=1\), we define the effective number of active features as \[ d_{\mathrm{eff}} = \frac{1}{\sum_{j=1}^d\omega_j^2}. \] This quantity equals \(d\) under uniform normalised feature weights and approaches 1 when the weighted geometry collapses onto a single feature.

\paragraph{External clustering agreement.} Where external labels are available, we report the Adjusted Rand Index (ARI) as a diagnostic measure of agreement between the clustering solution and the reference labels. ARI was not used to optimise centroid locations, feature weights, or counterfactuals,  but it provides useful context when interpreting the resulting cluster structure.

\section{Results and discussion}
\label{sec:results}

This section analyses the results of our experiments, following the sampling procedure described in Section~\ref{sec:experimental_evaluation}, repeated across all clustering modes, actionability masks, and contraction settings.

\subsection{Validity comparison}
\label{sec:geometric_validity_results}

Table~\ref{tab:vardakas_voronoi_validity_main} compares pairwise bisecting-hyperplane counterfactuals of CFCLUST with the VoICE unweighted (see Algorithm~\ref{alg:baseline_cf}). VoICE achieves 100.0\% validity by construction, since it optimises directly over the target Voronoi cell. This experiment tests whether CFCLUST also satisfies this criterion when more than two clusters are present, noting that for two-cluster solutions the pairwise bisector and the target Voronoi cell coincide.

Across $26{,}335{,}500$ factual--target comparisons over ten datasets, CFCLUST produced $21{,}221{,}926$ valid counterfactuals and $5{,}113{,}574$ invalid, corresponding to pooled validity of $80.58\%$ and a failure rate of $19.42\%$. Because Diabetes contributes a substantial majority of all comparisons, the macro-averaged validity across datasets was also calculated and was 83.26\%. 

Among invalid counterfactuals, the failure-weighted mean repair cost was 0.870, and the mean relative cost underestimation was 15.18\%. For Wholesale Customers, the corresponding relative underestimation was 21.86\%.

\begin{table*}[!t]
\centering
\scriptsize
\setlength{\tabcolsep}{5pt}
\renewcommand{\arraystretch}{1.05}
\caption{Validity (Val) of CFCLUST against unweighted, least cost Voronoi-region based counterfactuals. All results use 50 independent $k$-means iterations. Comparisons are factual--target pairs generated across those iterations.}
\label{tab:vardakas_voronoi_validity_main}
\begin{tabular}{lrcccc}
\toprule
&
&
\multicolumn{2}{c}{\textbf{CFCLUST}} &
\multicolumn{2}{c}{\textbf{VoICE}} \\
\cmidrule(lr){3-4} \cmidrule(lr){5-6}
\textbf{Dataset}&\textbf{Comparisons} &
\textbf{Val (\%)} &
\textbf{Repair Cost} &
\textbf{Val (\%)} &
\textbf{Repair Cost} \\
\midrule
Iris & 15,000 & 72.6 & 0.776 & 100.0 & 0.000 \\
Wine & 17,800 & 85.9 & 0.460 & 100.0 & 0.000 \\
Palmer Penguins & 34,200 & 73.8 & 0.427 & 100.0 & 0.000 \\
Breast Cancer & 28,450 & 100.0 & 0.000 & 100.0 & 0.000 \\
Wholesale Customers & 44,000 & 78.8 & 1.477 & 100.0 & 0.000 \\
Diabetes & 25,368,000 & 80.8 & 0.887 & 100.0 & 0.000 \\
Obesity & 633,300 & 71.9 & 0.463 & 100.0 & 0.000 \\
German Credit & 50,000 & 100.0 & 0.000 & 100.0 & 0.000 \\
Heart Failure & 14,950 & 100.0 & 0.000 & 100.0 & 0.000 \\
Student Performance & 129,800 & 68.8 & 0.521 & 100.0 & 0.000 \\
\midrule
\textbf{Overall} 
& \textbf{26{,}335{,}500}
& \textbf{80.6} 
& \textbf{0.870}
& \textbf{100.0} 
& \textbf{0.000} \\
\bottomrule
\end{tabular}
\end{table*}

\subsection{VoICE least-cost evaluation}
\label{sec:result_least_cost}

Table~\ref{tab:least_cost_primary} summarises the performance of VoICE (Algorithm~\ref{alg:voice_cf}) across the primary evaluation datasets. Nearly all features change under its least-cost optimisation regardless of dataset (99.4--100.0\%), since the method is not constrained to a small feature subset. Intervention cost varies considerably across datasets, from 0.496 (Breast Cancer) to 7.532 (Wholesale Customers). Notably, directional tolerance does not follow the same ordering as cost. Breast Cancer combines the lowest intervention cost with the highest tolerance, whereas Wholesale Customers combines the highest cost with the lowest mean tolerance. This indicates that low-cost counterfactuals are not necessarily the most robust, so cost and tolerance should be considered jointly rather than relying on either alone.

\begin{table}[!t]
\centering
\scriptsize
\setlength{\tabcolsep}{2pt}
\caption{VoICE (Algorithm~\ref{alg:voice_cf}) performance on the primary evaluation datasets. All features are set as actionable, and results are averaged across the three modes (Unweighted $k$-means, Ranked $k$-means, and Ranked + weighted SHARK). Feasibility was 100.0\% across all datasets and modes.}
\label{tab:least_cost_primary}
\begin{tabular}{lcccc}
\toprule
\textbf{Dataset} &
\textbf{Changed} &
\textbf{Cost}&
\textbf{Mean directional} &
\textbf{Runtime} \\
&\textbf{features}&$\boldsymbol{D_{\omega}(a, z^\star)}$&\textbf{tolerance ($\boldsymbol{\tau})$}&\textbf{(ms)}\\
\midrule
Iris& 99.8\% & 1.258 & 2.144 & 3.1 \\
Wine& 99.4\% & 0.528 & 1.402 & 8.0 \\
Palmer Penguins & 100.0\% & 0.874 & 1.927 & 2.8 \\
Breast Cancer & 100.0\% & 0.496 & 2.812 & 9.4 \\
Wholesale Customers & 99.4\% & 7.532 & 1.271 & 3.8 \\
\bottomrule
\end{tabular}
\end{table}

\subsection{Parsimony evaluation}

Table~\ref{tab:parsimony_primary} summarises the performance of VoICE's most-parsimonious counterfactuals across the primary evaluation datasets. To facilitate comparison across datasets with different dimensionalities, we report both the raw minimal intervention cardinality $r^\star$ and its feature-normalised version $r^\star/d$, where $d$ is the number of features. We do so because raw cardinality can be misleading. For instance, Breast Cancer exhibits the highest mean $r^\star$ but its feature-normalised cardinality is the lowest among all datasets. This indicates that relatively few of its available features are required to achieve a feasible cluster transition.

\begin{table}[!t]
\centering
\scriptsize
\setlength{\tabcolsep}{4pt}
\renewcommand{\arraystretch}{1.05}
\caption{VoICE parsimonious performance on the primary evaluation datasets. All features are set as actionable, results are averaged across the three modes (Unweighted $k$-means, Ranked $k$-means, and Ranked + weighted SHARK), and $r^\star$ is the minimal intervention cardinality. Feasibility was 100.0\% across all datasets and modes.}
\label{tab:parsimony_primary}
\begin{tabular}{lcccc}
\toprule
\textbf{Dataset} &
\textbf{Mean} &
\textbf{$\boldsymbol{r^\star/d}$} &
\textbf{Single-feature} &
\textbf{Mean directional} \\
& $\boldsymbol{r^\star}$&& \textbf{feasibility (\%)}&\textbf{tolerance ($\boldsymbol{\tau})$}\\
\midrule
Iris& 1.79 & 0.448 & 35.3  & 0.697 \\
Wine& 3.28 & 0.252 & 8.0 & 0.222 \\
Palmer Penguins& 1.85 & 0.462 & 36.0 & 0.756 \\
Breast Cancer& 3.88 & 0.129 & 5.3  & 0.099 \\
Wholesale Customers& 2.29 & 0.382 & 44.0  & 1.024 \\
\bottomrule
\end{tabular}
\end{table}

Table~\ref{tab:actionability_sensitivity} reports feasibility of the most-parsimonious counterfactual as an increasing proportion of features are randomly set non-actionable. Feasibility degrades at different rates across datasets. Breast Cancer remains highly feasible even at 75\% non-actionable features (88.7\%), whereas Iris and Palmer Penguins drop sharply under the same constraint (24.0\% and 24.7\%). This loosely tracks the feature-normalised cardinality $r^\star/d$ from Table~\ref{tab:parsimony_primary}.

\begin{table}[!t]
\centering
\scriptsize
\caption{Actionability sensitivity for most-parsimonious counterfactuals on the primary evaluation datasets. Entries show feasibility rates under increasingly restrictive random non-actionability masks, averaged across the three modes with contracted target regions.}
\label{tab:actionability_sensitivity}
\begin{tabular}{p{3.0cm}cccc}
\toprule
&\multicolumn{4}{c}{\textbf{Non-actionable features (\%)}}\\
\cline{2-5}
\textbf{Dataset} &
\textbf{0\%} &
\textbf{25\%} &
\textbf{50\%} &
\textbf{75\%} \\
\midrule
Iris& 100.0 & 93.3 & 66.7 & 24.0 \\
Wine& 100.0 & 100.0 & 98.7 & 39.3 \\
Palmer Penguins& 100.0 & 98.7 & 72.0 & 24.7 \\
Breast Cancer& 100.0 & 100.0 & 100.0 & 88.7 \\
Wholesale Customers& 100.0 & 94.7 & 86.0 & 60.0\\
\bottomrule
\end{tabular}
\end{table}

\subsection{Contraction-based robustness}

Table~\ref{tab:alpha_primary_by_algorithm} reports $\alpha^\star_t$ (the smallest contraction factor retaining all observed data points in a target cluster), and the median point-wise contraction score for each dataset under both Unweighted $k$-means and Ranked + weighted SHARK modes of VoICE (Algorithm~\ref{alg:voice_cf}). Breast Cancer is essentially invariant to mode, with median $\alpha^\star_t$ changing only marginally (0.989 to 0.992), consistent with its insensitivity to mode reported in Section~\ref{sec:result_least_cost}. Iris is the clearest exception: SHARK weighting sharply reduces its minimum $\alpha^\star_t$ (0.851 to 0.305) and its median point-wise score (0.079 to 0.000), indicating that its target regions become substantially more contractible once feature weights concentrate on a small number of features. Wholesale Customers shows the highest median point-wise scores among the primary datasets under both modes, with values of 0.274 for unweighted $k$-means and 0.417 for ranked and weighted SHARK. This indicates that its cluster members sit less centrally within their bounded target regions, consistent with the comparatively large intervention costs reported in Table~\ref{tab:least_cost_primary}.

\begin{table*}[!t]
\centering
\scriptsize
\setlength{\tabcolsep}{5pt}
\renewcommand{\arraystretch}{1.05}
\caption{Contraction-based robustness for the primary evaluation datasets, where $\alpha^\star_t$ is the smallest contraction factor retaining all observed data points in a target cluster.}
\label{tab:alpha_primary_by_algorithm}
\begin{tabular}{llcccc}
\toprule
\textbf{Dataset} &
\textbf{Mode} &
\textbf{Min $\alpha^\star_t$} &
\textbf{Median $\alpha^\star_t$} &
\textbf{Max $\alpha^\star_t$} &
\textbf{Median score} \\
\midrule

Iris
& Unweighted $k$-means
& 0.851 & 0.957 & 0.984 & 0.079 \\
& Ranked + weighted SHARK
& 0.305 & 0.861 & 0.928 & 0.000 \\

\midrule

Wine
& Unweighted $k$-means
& 0.891 & 0.915 & 0.945 & 0.127 \\
& Ranked + weighted SHARK
& 0.940 & 0.961 & 0.996 & 0.079 \\

\midrule

Palmer Penguins
& Unweighted $k$-means
& 0.677 & 0.958 & 0.989 & 0.105 \\
& Ranked + weighted SHARK
& 0.659 & 0.974 & 0.996 & 0.128 \\

\midrule

Breast Cancer
& Unweighted $k$-means
& 0.986 & 0.989 & 0.992 & 0.094 \\
& Ranked + weighted SHARK
& 0.992 & 0.992 & 0.992 & 0.079 \\

\midrule

Wholesale Customers
& Unweighted $k$-means
& 0.548 & 0.975 & 0.981 & 0.274 \\
& Ranked + weighted SHARK
& 0.931 & 0.991 & 0.999 & 0.417\\

\bottomrule
\end{tabular}
\end{table*}

\subsection{Weight-concentration stress tests}
\label{sec:Weight_concentration}

Table~\ref{tab:weight_concentration_diagnostics} reports feature-weight concentration and contraction for the primary and diagnostic stress-test datasets. The primary datasets retain multi-feature weighted geometries ($d_{\mathrm{eff}}$ from 2.53 to 26.21), whereas several diagnostic datasets collapse onto a single feature (Diabetes, German Credit, Heart Failure, $d_{\mathrm{eff}}=1.00$) or a small categorical block (Obesity, Student Performance, $d_{\mathrm{eff}}=2.00$). In these cases, near-zero $\alpha^\star_t$ reflects compactness in the learned weighted subspace rather than the full feature space. This justifies treating the mixed-type datasets as diagnostic stress tests rather than including them in the primary comparison, and motivates future work on categorical-aware distances, grouped feature weights, and regularisation of feature-weight concentration.

\begin{table*}[!t]
\centering
\scriptsize
\setlength{\tabcolsep}{4pt}
\renewcommand{\arraystretch}{1.05}
\caption{Feature-weight concentration and contraction under SHARK-weighted clustering, where $d_{\mathrm{eff}}=1/\sum_j \omega_j^2$ is the effective number of active features. Near-zero $\alpha^\star_t$ coincides with datasets where feature weights are most concentrated ($d_{\mathrm{eff}}$), suggesting weight concentration rather than genuine compactness.}
\label{tab:weight_concentration_diagnostics}
\begin{tabular}{lccclc}
\toprule
\textbf{Dataset} &
\textbf{Max $\omega$} &
\textbf{$d_{\mathrm{eff}}$} &
\textbf{Top-2 mass} &
\textbf{Dominant feature/block} &
\textbf{$\alpha^\star_t$ range} \\
\midrule
\multicolumn{3}{l}{\textit{Evaluation role: Primary}}\\
Iris
& 0.452
& 2.53
& 0.880
& petal width
& 0.305--0.928 \\

Wine
& 0.150
& 10.40
& 0.295
& OD280/OD315
& 0.940--0.996 \\

Palmer Penguins
& 0.358
& 3.58
& 0.629
& flipper length
& 0.659--0.996 \\

Breast Cancer
& 0.065
& 26.21
& 0.122
& concave points mean
& 0.992--0.992 \\

Wholesale Customers
&0.259
&4.91
&0.510
&Grocery
&0.931--0.999\\

\midrule
\multicolumn{3}{l}{\textit{Evaluation role: Diagnostic stress test}}\\
Diabetes
& 1.000
& 1.00
& 1.000
& Stroke
& 0.000 \\

Obesity
& 0.500
& 2.00
& 1.000
& CALC block
& 0.000 \\

German Credit
& 1.000
& 1.00
& 1.000
& credit-history dummy
& 0.000 \\

Heart Failure
& 1.000
& 1.00
& 1.000
& smoking
& 0.000 \\

Student Performance
& 0.500
& 2.00
& 1.000
& Fjob block
& 0.000 \\
\bottomrule
\end{tabular}
\end{table*}

\section{Conclusion}

This paper introduced VoICE, a geometric framework for generating counterfactual explanations in feature-weighted $k$-means clustering by projecting onto weighted Voronoi regions rather than pairwise cluster boundaries. The framework combines homothetic contraction, ranked actionable feature subsets, and directional counterfactual ranges within a unified optimisation formulation, and guarantees validity by construction, unlike pairwise-boundary approaches such as CFCLUST~\cite{vardakas2025counterfactual}.

Across the primary evaluation datasets, VoICE produced feasible counterfactuals with well-behaved weighted geometries. The results show that intervention cost, cardinality, and directional tolerance capture distinct and complementary aspects of explanation quality, feature-normalised cardinality $r^\star/d$ revealed that datasets with the largest raw intervention cardinality (e.g.\ Breast Cancer) could nonetheless be the most parsimonious relative to their dimensionality, and low-cost counterfactuals were not always the most robust. The diagnostic stress-test datasets further showed that feature-weight concentration can cause contraction statistics to reflect a collapsed, low-dimensional subspace rather than genuine multivariate compactness, an important practical caveat for mixed-type data.

A key direction for future work concerns mixed-type datasets with binary, ordinal, or one-hot encoded features, where concentrated weights can collapse the geometry onto a small subset of features. Categorical-aware distances and grouped or regularised feature weights are promising directions here.

\ifarxiv
    \bibliography{references}
\fi

\end{document}